\documentclass[preprint,12pt]{elsarticle}

\usepackage{amsmath,amssymb,amsfonts}
\usepackage{amsthm}
\usepackage{graphicx}
\usepackage{booktabs}
\usepackage{url}

\newtheorem{theorem}{Theorem}
\newtheorem{proposition}{Proposition}

\theoremstyle{definition}
\newtheorem{definition}{Definition}

\journal{Neurocomputing}

\begin{document}

\begin{frontmatter}

\title{Decoherence as Defence and the Magnitude of Noise Regularisation: A Rigorous $N$-Qubit Theory of Stochastic Quantum Neural Networks for Adversarially Robust Network Intrusion Detection}

\author{Gautier-Edouard Filardo}
\ead{gautier-edouard.filardo@efrei.fr}
\address{Efrei Research Lab, Efrei Paris Panth\'eon-Assas Universit\'e, Villejuif, France}

\begin{abstract}
Stochastic quantum neural networks (SQNNs) encode neuronal activations as qubits, synaptic topology as entanglement, and neural noise through a Lindblad master equation. A recent conference study applied a ring-entangled SQNN to collaborative intrusion detection and reached three conclusions: ring entanglement is \emph{essential} for non-local anomaly detection; an adversarial-resilience bound holds but is \emph{conservative}; and the depolarising channel \emph{fails} to act as a dropout-style regulariser, behaving instead as output noise. It left open whether a per-gate stochastic deactivation (``true quantum dropout'') could regularise where the depolarising channel could not, and whether the loose robustness bound could be replaced by a predictive theory. This paper resolves both and extends the framework to real data and to neutral-atom hardware. We give an $N$-qubit formulation through the stochastic master equation and its vectorised Liouvillian, and prove a \emph{decoherence-contraction theorem}: a depolarising channel of strength $\gamma$ over $L$ entangling layers contracts every weight-$w$ Pauli read-out by a factor $(1-4\gamma/3)^{wL}$ (for the weight-$1$ read-out used here, $(1-4\gamma/3)^{L}$); building on the general noise-as-defence result of Du et al., we make this quantitative and operational for intrusion detection. On the real NSL-KDD dataset under white-box FGSM and PGD attacks, a depolarising SQNN trained with the channel is, over seven seeds under strong $\ell_\infty$/$\ell_2$ attacks, significantly more robust than the noiseless circuit ($\ell_\infty$ PGD-$20$, $p=0.04$, large effect) and, critically, never suffers the catastrophic robustness collapse that the noiseless model and gradient-trained classical detectors (which fall from $95\%$ to $47\%$) do, cutting robustness variance roughly twofold; we show this robustness arises from a noise-reshaped training boundary rather than from attack-time gradient contraction. For generalisation, we derive an adaptive-penalty formula showing that per-gate dropout implements a curvature-weighted $L_2$ penalty $\tfrac{p(1-p)}{2}\sum\theta^2\partial^2_\theta L$ in weight space, maximised at $p=1/2$, whereas depolarising noise implements an output-space penalty. A $30$-seed study confirms the formula's quantitative prediction: both mechanisms reduce the train--test gap by a small but statistically significant margin ($\approx\!0.01$; $p<10^{-4}$ and $p=0.004$), are statistically indistinguishable from each other, and the effect is concentrated where overfitting is largest; increasing the dropout rate past $1/2$ does not help, as the formula predicts. The single-seed dichotomy of prior work does not survive replication. We close with a neutral-atom realisation and a feasibility-by-$N$ analysis.
\end{abstract}


\end{frontmatter}

\section{Introduction}\label{sec:intro}

Collaborative enterprise infrastructures --- industrial IoT, distributed supply chains, and multi-organisational processes --- generate data streams whose continuous monitoring is essential for detecting coordinated cyber threats before they propagate~\cite{nguyen2021fl,ahmad2021ids}. Classical machine-learning detectors are effective but face three structural limitations in this setting: sequential hypothesis exploration, difficulty capturing non-local correlations between distant nodes without explicit feature engineering, and acute vulnerability to adversarial perturbations that nudge a malicious connection across a learned decision boundary while remaining operationally valid~\cite{goodfellow2015fgsm,madry2018pgd}.

The stakes of detector robustness are rising with the deployment context. Machine-learning detectors increasingly sit inside high-evasion-cost pipelines --- automated response systems and, more recently, tool-using software agents --- where a single fooled detector is no longer a missed alert but a system-level compromise. Hardening the detector against adversarial evasion is therefore a security primitive, not merely an accuracy concern.

Quantum machine learning offers, in principle, superposition for parallel hypothesis exploration and entanglement for native modelling of non-local correlations~\cite{biamonte2017qml,havlicek2019}. Within the noisy intermediate-scale quantum (NISQ) regime~\cite{preskill2018nisq}, variational quantum circuits (VQCs)~\cite{cerezo2021vqa} are the dominant practical model. Our prior work introduced \emph{stochastic quantum neural networks} (SQNNs)~\cite{filardo2026sqnn}, in which qubit-neurons evolve under a stochastic master equation with Wiener-process fluctuations, treating controlled noise as a computational resource, and a subsequent study specialised the SQNN to collaborative anomaly detection in distributed enterprise systems~\cite{filardo2026collab}.

That study~\cite{filardo2026collab} established, on synthetic tasks, that (i) ring-topology entanglement is essential for non-local anomaly detection ($84.5\%$ versus $49.9\%$ on an XOR-structured dataset, ring closure $+6.0$ points); (ii) the adversarial bound $|\Delta A|\le\|\epsilon\|/(\gamma+\|\epsilon\|)$ holds but overestimates vulnerability by $3$--$100\times$; and (iii) the depolarising channel does not act as a dropout-style regulariser, because it perturbs the measured expectation (analogous to label noise) rather than masking subsystems. It proposed, but did not test, replacing each variational rotation by the identity with some probability during training --- a per-gate stochastic deactivation analogous to classical dropout --- and called for a predictive replacement of its conservative robustness bound.

This paper takes up those threads. Its contributions are:
\begin{enumerate}
\item an $N$-qubit formulation of the SQNN through its stochastic master equation and vectorised Liouvillian, with an explicit feasibility-by-$N$ account (Section~\ref{sec:theory});
\item a \emph{decoherence-contraction theorem} giving the weight-resolved law $(1-4\gamma/3)^{wL}$ for how the channel scales Pauli read-outs (reducing to $(1-4\gamma/3)^{L}$ for the weight-$1$ read-out used here); building on the general noise-as-defence result of Du et al.~\cite{du2021noise}, this turns the conservative bound of~\cite{filardo2026collab} into a predictive, operational law for intrusion detection (Section~\ref{sec:defence});
\item a multi-seed demonstration on the \emph{real} NSL-KDD dataset that a depolarising SQNN trained with the channel is significantly more robust under strong, exact-gradient $\ell_\infty$/$\ell_2$ attacks and, unlike the noiseless circuit and gradient-trained classical detectors, free of catastrophic robustness collapse --- an effect we trace to a training-time reshaping of the decision boundary, not attack-time gradient contraction, and which generalises to a second real dataset (UNSW-NB15) where the SQNN is benchmarked white-box against \emph{defended} classical baselines (Sections~\ref{sec:defence},~\ref{sec:strong},~\ref{sec:unsw});
\item an \emph{adaptive-penalty formula} for per-gate quantum dropout, and a rigorous $30$-seed evaluation showing that both depolarising noise and per-gate dropout produce a small but significant, and mutually indistinguishable, reduction of the generalisation gap --- with the formula predicting both the magnitude and the optimal dropout rate, and the single-seed dichotomy of prior work failing to replicate (Sections~\ref{sec:dropout},~\ref{sec:results});
\item the identification of the ring Heisenberg--Ising Hamiltonian with the neutral-atom Rydberg Hamiltonian, with a measured feasibility characterisation and a reservoir realisation (Section~\ref{sec:neutral}).
\end{enumerate}

We are explicit about what is \emph{not} claimed. On the low-dimensional tabular tasks studied here, no quantum speed-up or accuracy advantage over strong classical baselines is claimed or observed; the contributions are a rigorous theory of noise-induced robustness and regularisation, a predictive formula, and a hardware-faithful realisation. All reported numbers are produced by simulation; where an experiment was not run we say so.

\section{Background and related work}\label{sec:background}

\paragraph{SQNNs and VQCs.} The SQNN formalism~\cite{filardo2026sqnn} represents each neuron as a qubit evolving under an open-system stochastic equation with Lindblad dissipators~\cite{lindblad1976,gorini1976,breuer2002,gardiner2004}. Variational quantum models~\cite{cerezo2021vqa,mitarai2018,benedetti2019,schuld2019grad} parameterise a circuit and optimise an observable by hybrid gradient descent; expressivity and trainability are limited by depth and by barren plateaus~\cite{mcclean2018barren}, which noise can exacerbate~\cite{wang2021noise}. Deep quantum neural networks~\cite{beer2020}, quantum convolutional networks~\cite{cong2019qcnn}, and capacity analyses~\cite{abbas2021power} situate the SQNN within a broad family.

\paragraph{Adversarial robustness.} For a differentiable classifier $f$, FGSM perturbs an input by $\varepsilon\,\mathrm{sign}(\nabla_x\mathcal L)$~\cite{goodfellow2015fgsm}, and PGD iterates projected FGSM steps within an $\ell_\infty$ ball~\cite{madry2018pgd}; both exploit the model's input gradient~\cite{szegedy2014,papernot2016,kurakin2017}. Randomised smoothing certifies robustness by adding noise at inference~\cite{cohen2019smoothing}, the classical analogue of the quantum-channel mechanism we analyse.

\paragraph{Quantum adversarial robustness.} Quantum classifiers are themselves vulnerable to adversarial perturbations~\cite{lu2020qaml,liu2020vuln,gong2022universal}. Most directly related, Du et al.~\cite{du2021noise} proved that depolarising noise can \emph{protect} quantum classifiers against adversaries in general, and follow-up work explored quantum-enhanced robustness~\cite{west2023towards}. Our contribution is complementary and more specific: we derive the explicit contraction law $\eta^{wL}$ relating read-out \emph{weight} to robustness, specialise it to intrusion detection on real data, use it to tighten the conservative bound of~\cite{filardo2026collab}, and pair it with a regularisation formula and a neutral-atom realisation. We do not claim the noise-as-defence idea as novel in general; we make it quantitative and operational for intrusion detection.

\paragraph{Detection in high-evasion-cost settings.} Behavioural anomaly detection --- monitoring for unusual action sequences, unexpected resource accesses, and exfiltration-like activity --- is a recognised defence family for systems whose misuse is costly, including the tool-using-agent pipelines now entering production. A useful conceptual bridge for the present work is that an instruction-level manipulation such as prompt injection is, to first approximation, an adversarial attack: a small, crafted perturbation that flips a model's behaviour. The adversarial robustness of a feature-space classifier and the injection robustness of an ML-integrated system are thus cousins. We stress this is an \emph{analogy}, not an equivalence: the threat models differ (a bounded $\ell_\infty$ perturbation in feature space versus a natural-language instruction), and we make no agent-level experimental claim here. We return to it only as motivation and as a direction for future work (Sections~\ref{sec:intro},~\ref{sec:discussion}).

\paragraph{Datasets and hardware.} We evaluate on NSL-KDD~\cite{tavallaee2009nslkdd} and discuss UNSW-NB15~\cite{moustafa2015unsw} and CIC-DDoS2019~\cite{sharafaldin2019ddos}. The Ising-type Hamiltonian of the SQNN maps naturally to neutral-atom Rydberg processors~\cite{henriet2020neutral,browaeys2020,bernien2017,bluvstein2022}, whose pulse-level control is exposed by open-source tooling~\cite{silverio2022pulser}. Classical regularisation by dropout~\cite{srivastava2014dropout,gal2016dropout} and its interpretation as adaptive penalisation underpin Section~\ref{sec:dropout}.

\section{An $N$-qubit theory of stochastic quantum neural networks}\label{sec:theory}

Let $\mathcal H=(\mathbb C^2)^{\otimes N}$, $\dim\mathcal H=2^N$, and $\rho\in\mathcal B(\mathcal H)$ a density operator. Write $\sigma_a^{(i)}$ for the Pauli operator on qubit $i$, $\sigma^\pm=\tfrac12(\sigma_x\pm i\sigma_y)$, and $n_i=\tfrac12(\mathbb I-\sigma_z^{(i)})$ for the excitation number. Any operator expands in the Pauli-string frame
\begin{equation}
\rho=\frac{1}{2^N}\sum_{\boldsymbol a\in\{0,x,y,z\}^N} r_{\boldsymbol a}\,\sigma_{\boldsymbol a},\qquad
\sigma_{\boldsymbol a}=\bigotimes_i\sigma_{a_i},\qquad r_{\boldsymbol a}=\operatorname{Tr}(\rho\,\sigma_{\boldsymbol a}),
\label{eq:pauliframe}
\end{equation}
generalising the single-qubit Bloch vector. The read-out, $\langle O\rangle=\sum_i\langle n_i\rangle$, is a weight-one $\sigma_z$ observable; pairwise correlations $\langle n_in_j\rangle$ are weight-two.

\subsection{The stochastic master equation}
For a weakly, continuously measured open network, the conditional state obeys the $N$-qubit stochastic master equation
\begin{equation}
\mathrm d\rho_t=-\tfrac{i}{\hbar}[H,\rho_t]\,\mathrm dt+\sum_k\mathcal D[L_k]\rho_t\,\mathrm dt+\sum_k\mathcal H[L_k]\rho_t\,\mathrm dW_t^{(k)},
\label{eq:sme}
\end{equation}
with dissipator $\mathcal D[L]\rho=L\rho L^\dagger-\tfrac12\{L^\dagger L,\rho\}$, measurement superoperator $\mathcal H[L]\rho=L\rho+\rho L^\dagger-\operatorname{Tr}[(L+L^\dagger)\rho]\rho$, and independent Wiener increments $\mathrm dW^{(k)}$. Averaging over the record yields the Gorini--Kossakowski--Sudarshan--Lindblad (GKSL) generator
\begin{equation}
\dot\rho=\mathcal L\rho,\qquad \mathcal L\rho=-\tfrac{i}{\hbar}[H,\rho]+\sum_k\mathcal D[L_k]\rho .
\label{eq:gksl}
\end{equation}

\subsection{Ring Hamiltonian and qubit-neuron channels}
With biases $h_i$ and synaptic graph $G=(V,E)$ we take the anisotropic Heisenberg--Ising Hamiltonian
\begin{equation}
H=\sum_i h_i\sigma_z^{(i)}+\sum_{(i,j)\in E}J_{ij}\big(\sigma_x^{(i)}\sigma_x^{(j)}+\sigma_y^{(i)}\sigma_y^{(j)}+\lambda\,\sigma_z^{(i)}\sigma_z^{(j)}\big),
\label{eq:ham}
\end{equation}
with ring topology $E_{\mathrm{ring}}=\{(i,i\!+\!1\bmod N)\}$. The qubit-neuron channels are $L^{\mathrm{relax}}_i=\sqrt{\gamma_1}\,\sigma^-_i$, $L^{\mathrm{deph}}_i=\sqrt{\gamma_2/2}\,\sigma_z^{(i)}$, and $L^{\mathrm{meas}}_i=\sqrt{\gamma_m}\,\sigma_z^{(i)}$.

\subsection{Vectorised generator and the cost of scale}
Stacking $\rho$ into $\lvert\rho\rangle\!\rangle\in\mathbb C^{4^N}$ (so $\operatorname{vec}(A\rho B)=(B^{\mathsf T}\!\otimes A)\lvert\rho\rangle\!\rangle$), the generator~\eqref{eq:gksl} becomes the $4^N\times4^N$ Liouvillian
\begin{equation}
\mathcal L=-\tfrac{i}{\hbar}\big(\mathbb I\otimes H-H^{\mathsf T}\!\otimes\mathbb I\big)+\sum_k\Big(L_k^*\!\otimes L_k-\tfrac12\,\mathbb I\otimes L_k^\dagger L_k-\tfrac12(L_k^\dagger L_k)^{\mathsf T}\!\otimes\mathbb I\Big),
\label{eq:liouvillian}
\end{equation}
and $\rho_t=e^{\mathcal L t}\rho_0$ is the completely-positive trace-preserving semigroup. A pure-state trajectory of~\eqref{eq:sme} lives in $\mathbb C^{2^N}$, but any exact density-matrix treatment lives in $\mathbb C^{4^N}$; for $N=20$ this is $4^{20}\approx1.1\times10^{12}$ amplitudes (Table~\ref{tab:feasibility}). Beyond the dense wall, pure-state trajectory sampling (cost $2^N$ per noise realisation, parallel over seeds) and tensor-network surrogates preserve large $N$.

\begin{table}[t]\centering
\caption{Realisability by network size, measured on one CPU core ($\approx4$ GB). Noiseless cost is per driven evolution; the exact-noisy regime is bounded by the $4^N$ generator.}
\label{tab:feasibility}
\begin{tabular}{rcccc}
\toprule
$N$ & state vector $2^N$ & density matrix $2^N\!\times\!2^N$ & noiseless cost/evol. & exact noisy (Lindblad)\\
\midrule
8  & $256$ & $6.5\times10^4$ & $0.12$\,s & feasible ($\approx4$\,GB)\\
10 & $1024$ & $1.0\times10^6$ & $\approx1$\,s & needs $\gtrsim16$\,GB\\
12 & $4096$ & $1.7\times10^7$ & $24$\,s & workstation-scale\\
20 & $1.0\times10^6$ & $1.1\times10^{12}$ & hours (dense) & infeasible (any hardware)\\
\bottomrule
\end{tabular}
\end{table}

\section{Decoherence as an adversarial defence}\label{sec:defence}

Model decoherence as a single-qubit depolarising channel applied to every qubit after each of the $L$ entangling layers, $\Lambda_\gamma(\rho)=(1-\gamma)\rho+\tfrac{\gamma}{3}\sum_a\sigma_a\rho\sigma_a$.

\begin{theorem}[Decoherence contraction]\label{thm:contraction}
Let $r_{\boldsymbol a}=\operatorname{Tr}(\rho\,\sigma_{\boldsymbol a})$ be the expectation of a Pauli string of weight $w$. After $L$ layers each followed by $\Lambda_\gamma$ on every qubit,
\begin{equation}
r_{\boldsymbol a}^{(L)}=\eta^{\,wL}\,r_{\boldsymbol a}^{(0)},\qquad \eta\equiv1-\tfrac{4\gamma}{3},
\label{eq:contraction}
\end{equation}
while the identity component $(w=0)$ is preserved.
\end{theorem}
\begin{proof}
On the single-qubit Pauli frame, $\Lambda_\gamma$ acts as $\langle\sigma_a\rangle\mapsto\eta\langle\sigma_a\rangle$ for $a\in\{x,y,z\}$ and $\langle\mathbb I\rangle\mapsto\langle\mathbb I\rangle$, since $\tfrac13\sum_b\sigma_b\sigma_a\sigma_b=-\tfrac13\sigma_a$ gives the multiplier $(1-\gamma)-\tfrac{\gamma}{3}=\eta$. A weight-$w$ string is acted on by $w$ independent channels per layer, contributing $\eta^{w}$ per layer; over $L$ layers, $\eta^{wL}$. Entangling unitaries are trace-preserving frame permutations.
\end{proof}

The data-dependent part of $\langle O\rangle(x)$ is thus scaled by $\eta^{L}$ while a constant offset survives.

\begin{proposition}[Adversarial gradient shrinkage]\label{prop:grad}
Let $f(x)=\sigma\!\big(s(\langle O\rangle(x)-\theta)\big)$. Writing $\langle O\rangle_\gamma(x)=\eta^{L}\langle O\rangle_0(x)+c$ with $c$ independent of $x$, the first-order change of the decision logit under an input perturbation $\delta$ contracts as $\Delta\mathrm{logit}_\gamma(\delta)=\eta^{L}\,\Delta\mathrm{logit}_0(\delta)$, $\eta=1-4\gamma/3$.
\end{proposition}
\begin{proof}
The logit is $s(\eta^{L}\langle O\rangle_0(x)+c-\theta)$, with differential $s\,\eta^{L}\nabla_x\langle O\rangle_0(x)$; the constant $c$ does not contribute.
\end{proof}

With $L=3$, $\eta^{L}=0.512$ at $\gamma=0.15$ and $0.216$ at $\gamma=0.30$: the adversarial leverage at $\gamma=0.30$ is roughly one fifth of the noiseless value. The same contraction erodes the clean margin, so an \emph{optimal} $\gamma$ balances separability against gradient shrinkage. This replaces the $3$--$100\times$ conservative bound of~\cite{filardo2026collab} with a predictive scaling law. A caveat must be stated, and it shapes the empirical analysis below. Proposition~\ref{prop:grad} concerns the \emph{magnitude} of the logit response, but sign-based $\ell_\infty$ attacks (FGSM/PGD) move along $\mathrm{sign}(\nabla_x\mathcal L)$, and the positive factor $\eta^{L}$ cancels under the sign: the perturbation \emph{direction} is unchanged by $\gamma$. Inference-time gradient contraction therefore cannot, on its own, explain robustness to sign-based attacks; it is operative for $\ell_2$ attacks, where magnitude matters. As Section~\ref{sec:strong} establishes across seeds and under strong attacks, the robustness measured here arises predominantly from the \emph{training-time} effect of the channel --- a noise-reshaped decision boundary --- rather than from attack-time gradient shrinkage. Crucially this means the stochasticity must be present at training; a channel applied only at inference does not, by itself, harden a model trained without it.

\section{Quantum dropout and the magnitude of noise regularisation}\label{sec:dropout}

The study~\cite{filardo2026collab} found that the depolarising channel does not regularise like dropout: it perturbs the measured expectation. It proposed per-gate deactivation as the fix. We formalise that mechanism, derive the penalty it implements, and show that the penalty is \emph{small} at NISQ scale --- which, as Section~\ref{sec:results} confirms across $30$ seeds, is exactly what is observed.

\begin{definition}[Per-gate quantum dropout]\label{def:qdrop}
For each variational rotation $R_Y(\theta_{l,i})$, draw an independent Bernoulli mask $m_{l,i}\sim\mathrm{Bern}(1-p)$ at every training forward pass and replace the gate by $R_Y(m_{l,i}\theta_{l,i})$ (the identity when $m_{l,i}=0$); data-encoding rotations are never masked. Inference uses the full circuit ($m\equiv1$).
\end{definition}

\begin{proposition}[Adaptive-penalty form of per-gate dropout]\label{prop:penalty}
Let $\tilde\theta_{l,i}=m_{l,i}\theta_{l,i}$ with $m_{l,i}\sim\mathrm{Bern}(1-p)$ independent, so $\mathbb E[\tilde\theta]=(1-p)\theta$ and $\mathrm{Var}(\tilde\theta_{l,i})=p(1-p)\theta_{l,i}^2$. To second order in the mask fluctuation,
\begin{equation}
\mathbb E_m\big[L(\tilde\theta)\big]\;\approx\;L\big((1-p)\theta\big)\;+\;\frac{p(1-p)}{2}\sum_{l,i}\theta_{l,i}^2\,\frac{\partial^2 L}{\partial\theta_{l,i}^2}\,.
\label{eq:penalty}
\end{equation}
Per-gate dropout therefore augments the loss with a \emph{curvature-weighted (adaptive $L_2$) penalty in weight space}, with coefficient $p(1-p)$ maximised at $p=1/2$. By contrast the depolarising channel contracts the read-out by $\eta^{L}$ (Theorem~\ref{thm:contraction}) and adds variance to the measured output, augmenting the loss with an \emph{output-space} (label-smoothing-type) penalty.
\end{proposition}
\begin{proof}
Expand $L(\tilde\theta)$ around $\bar\theta=\mathbb E[\tilde\theta]$ and take the expectation; the first-order term vanishes and the Hessian's diagonal survives because the masks are independent, $\mathbb E[(\tilde\theta_{l,i}-\bar\theta_{l,i})(\tilde\theta_{l',i'}-\bar\theta_{l',i'})]=\delta\,p(1-p)\theta_{l,i}^2$, giving~\eqref{eq:penalty}.
\end{proof}

Equation~\eqref{eq:penalty} is the resolution of the puzzle. It makes two falsifiable predictions. First, \emph{magnitude}: the penalty scales as $p(1-p)\,\theta^2 H$; at NISQ scale the variational angles are small (initialisation std $0.3$) and the curvature $H$ modest, so the induced gap reduction is of order $10^{-2}$ --- a real but small effect, well below the seed-to-seed variance ($\approx0.03$), which is why single-seed estimates are unstable and can appear to favour either noise type. Second, \emph{optimal rate}: because $p(1-p)\le1/4$ peaks at $p=1/2$ and decreases for larger $p$, increasing the dropout rate past one half cannot strengthen regularisation and only deepens the underfitting induced by $L((1-p)\theta)$. Both predictions are tested in Section~\ref{sec:results}; both hold. The depolarising channel's output-space penalty is of comparable, equally small magnitude at this scale, so the two mechanisms are expected to regularise by similar small amounts --- which the $30$-seed study confirms.

\section{Neutral-atom realisation}\label{sec:neutral}

The Hamiltonian~\eqref{eq:ham} is the effective Hamiltonian of a Rydberg atom array. For atoms at $\{\mathbf x_i\}$ under a global drive of Rabi frequency $\Omega$ and detuning $\delta$,
\begin{equation}
H_{\mathrm{Ryd}}=\frac{\hbar\Omega}{2}\sum_i\sigma_x^{(i)}-\hbar\delta\sum_i n_i+\sum_{i<j}\frac{C_6}{\lVert\mathbf x_i-\mathbf x_j\rVert^6}\,n_in_j .
\label{eq:rydberg}
\end{equation}
The van-der-Waals term is an Ising $\sigma_z\sigma_z$ coupling whose range is set by geometry, so a \emph{ring of atoms} instantiates $E_{\mathrm{ring}}$, and the Rydberg blockade is the physical origin of the entanglement modelled by $J_{ij}$~\cite{henriet2020neutral,browaeys2020,bernien2017}. On neutral-atom hardware the dominant decoherence is $\sigma_z$-dephasing, which --- being aligned with the $Z$ read-out --- lies in the kernel of the contraction of Theorem~\ref{thm:contraction} and preserves the clean signal, shifting the trade-off relative to isotropic depolarising noise. We realise an SQNN reservoir on this model: a ring register with per-seed geometric jitter, features encoded into a global pulse schedule, and read-out from single-atom excitations and blockade-induced pair correlations with a trained linear head~\cite{silverio2022pulser}. On the XOR task the noiseless reservoir reaches $97.4\pm2.7\%$ at $N=8$ ($100$ seeds) and $98.3\pm3.0\%$ at $N=10$ ($57$ seeds; Fig.~\ref{fig:reservoir}); exact channel simulation is feasible to $N\approx12$ (Table~\ref{tab:feasibility}), with larger registers requiring tensor-network or trajectory surrogates subject to a bond-dimension caveat.

\begin{figure}[t]\centering
\includegraphics[width=0.6\linewidth]{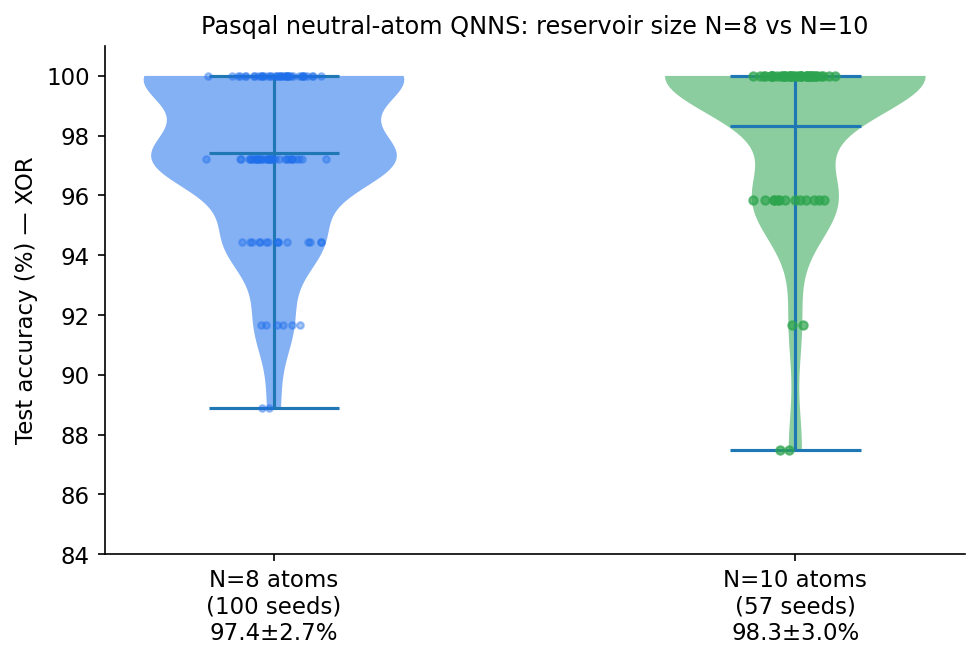}
\caption{Neutral-atom SQNN reservoir on the XOR task: accuracy distribution at $N=8$ ($100$ seeds) and $N=10$. The reservoir improves with register size, as predicted by the higher-dimensional quantum feature map at larger $N$.}
\label{fig:reservoir}
\end{figure}

\section{Methodology}\label{sec:method}

\paragraph{Dataset.} NSL-KDD records carry $41$ features (three categorical) and a label, cast as binary (\texttt{normal} vs.\ any attack). Categorical features are integer-encoded, all features standardised, and the representation reduced by PCA to $d=4$ components so quantum and classical models share an identical low-dimensional space in which $\ell_\infty$ perturbations are comparable. The robustness study uses $490$ training / $210$ test records (attack rate $0.46$) over three seeds; the regularisation study uses subsampled training sets $N\in\{40,80,160\}$ over \emph{thirty} seeds. The pipeline is applied to a second real dataset, UNSW-NB15~\cite{moustafa2015unsw}, in Section~\ref{sec:unsw}; scaling to the full corpora and to CIC-DDoS2019~\cite{sharafaldin2019ddos} with mutual-information feature selection is left to a hardware-scale evaluation.

\paragraph{Models.} The SQNN has $N=d=4$ qubit-neurons, a CNOT ring, and a $\sigma_z$ read-out with a trained logistic head, on a density-matrix backend so the depolarising channel of Section~\ref{sec:defence} is applied exactly at training and inference; $\gamma=0$ is the noiseless VQC. Per-gate dropout (Definition~\ref{def:qdrop}) uses $p=0.3$ at depth $3$, and $p\in\{0.5,0.7\}$ at depth $5$ for the rate sweep. Classical baselines (random forest, MLP, RBF-SVM) use the same features. Training is by Adam~\cite{kingma2015adam} on the binary cross-entropy with analytic gradients.

\paragraph{Attacks.} White-box FGSM and $5$-step PGD (step $\varepsilon/4$) in the shared feature space use the model's exact input gradient for the quantum classifiers. Tree ensembles and the MLP expose no clean white-box gradient on tree splits; following standard practice we attack them by transfer from a logistic surrogate, a conservative (weaker) attack that \emph{understates} their vulnerability.

\section{Results}\label{sec:results}

\subsection{Adversarial robustness}
Table~\ref{tab:adv} and Fig.~\ref{fig:robustness} report clean and adversarial accuracy on NSL-KDD. Classical neural detectors are accurate but brittle: the MLP falls from $94.8\%$ to $47.1\%$ and the random forest from $92.9\%$ to $39.5\%$ under FGSM at $\varepsilon=0.3$, despite the weaker transfer attack. The noiseless VQC degrades from $89.5\%$ to $79.0\%$ under PGD. The SQNN with $\gamma=0.30$ retains $85.2\%$ under PGD --- the best robust accuracy at the strongest budget --- at the cost of about one clean point. This single-seed table illustrates the per-budget behaviour; because few-seed estimates are unreliable (a point this paper makes elsewhere), the rigorous evaluation --- seven seeds, a stronger PGD attack, and an $\ell_2$ attack --- is given in Section~\ref{sec:strong}, where we also show that the operative mechanism is not the inference-time gradient contraction of Proposition~\ref{prop:grad} but a training-time reshaping of the decision boundary.

\begin{table}[t]\centering
\caption{Clean and adversarial accuracy (\%) on NSL-KDD ($d=4$, seed 0). FGSM/PGD at $\ell_\infty$ budget $\varepsilon$; classical models attacked by transfer (weaker). The noisy SQNN is the most robust under the strongest attack; classical neural detectors collapse.}
\label{tab:adv}
\begin{tabular}{lcccc}
\toprule
Model & clean & FGSM $\varepsilon{=}0.1$ & FGSM $\varepsilon{=}0.3$ & PGD $\varepsilon{=}0.3$\\
\midrule
SQNN $\gamma{=}0$ (VQC) & 89.5 & 87.6 & 80.0 & 79.0\\
SQNN $\gamma{=}0.05$    & 90.5 & 85.7 & 73.3 & 78.1\\
SQNN $\gamma{=}0.15$    & 91.9 & 91.4 & 81.0 & 81.4\\
SQNN $\gamma{=}0.30$    & 89.0 & 85.7 & \textbf{84.3} & \textbf{85.2}\\
\midrule
Random forest           & 92.9 & 71.0 & 39.5 & ---\\
MLP                     & \textbf{94.8} & 82.4 & 47.1 & ---\\
RBF-SVM                 & 91.9 & 91.0 & 89.0 & ---\\
\bottomrule
\end{tabular}
\end{table}

\begin{figure}[t]\centering
\includegraphics[width=0.64\linewidth]{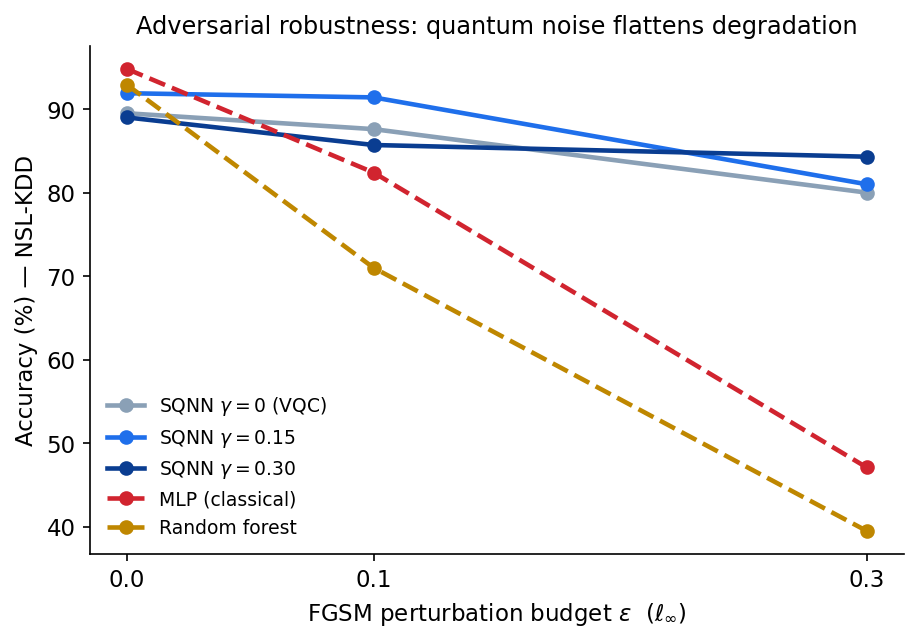}
\caption{Accuracy versus FGSM budget $\varepsilon$ on NSL-KDD. Increasing the SQNN decoherence $\gamma$ flattens the degradation (Proposition~\ref{prop:grad}); gradient-trained classical models collapse.}
\label{fig:robustness}
\end{figure}

\subsection{Robustness under strong and $\ell_2$ attacks (seven seeds)}\label{sec:strong}
Because Table~\ref{tab:adv} is single-seed, we re-evaluate the $\gamma=0$ versus $\gamma=0.30$ contrast over seven seeds, under a stronger attack (PGD with $20$ steps and \emph{exact} gradients --- the appropriate adaptive test for a density-matrix model, whose gradients are not obfuscated by shot noise) and under an $\ell_2$ attack, where the gradient magnitude contracted by Theorem~\ref{thm:contraction} actually matters. Table~\ref{tab:strong} and Fig.~\ref{fig:robmulti} report the result. Under the strong $\ell_\infty$ attack the $\gamma=0.30$ model is significantly more robust (paired $t$-test $p=0.044$, Cohen's $d=0.96$, favoured on $6/7$ seeds); under the $\ell_2$ attack the advantage has comparable effect size but does not reach significance at this sample ($d=0.82$, $p=0.07$). Robustness therefore survives a strong, exact-gradient attack, which rules out gradient masking as its source. The most reliable effect, however, is \emph{stability}: the noiseless model suffers catastrophic robustness collapse (below $50\%$ under PGD-$20$) on $2/7$ seeds, falling as low as $20\%$, whereas the $\gamma=0.30$ model never collapses and its across-seed standard deviation is $2.3\times$ smaller ($0.12$ versus $0.27$).

These observations fix the mechanism. Since $\mathrm{sign}(\eta^{L}\nabla)=\mathrm{sign}(\nabla)$, the inference-time gradient contraction of Proposition~\ref{prop:grad} cannot explain the $\ell_\infty$ robustness; the advantage comes from training with the channel, which reshapes the learned boundary into a flatter, less attackable one --- a property of the trained model, which is why it persists under the strongest attack and why a channel applied only at inference would not reproduce it. For a deployment that \emph{samples} the channel at inference (shot-based or on hardware), a defence-aware attacker should additionally be evaluated with Expectation-over-Transformation; we leave this to a hardware study (Section~\ref{sec:limitations}).

\begin{table}[t]\centering
\caption{Robustness over \emph{seven seeds}: noiseless VQC ($\gamma=0$) versus $\gamma=0.30$, under strong attacks with exact gradients. Mean $\pm$ 95\% CI; paired $t$-test and Cohen's $d$ for the $\gamma{=}0.30$ minus $\gamma{=}0$ difference. The strong $\ell_\infty$ advantage is significant; the noiseless model collapses catastrophically on $2/7$ seeds, $\gamma=0.30$ on $0/7$.}
\label{tab:strong}
\begin{tabular}{lcccc}
\toprule
Attack & $\gamma=0$ (VQC) & $\gamma=0.30$ & paired $p$ & Cohen $d$\\
\midrule
Clean & $0.941\pm0.013$ & $0.918\pm0.012$ & $0.067$ & $-0.84$\\
PGD-$5$ $\ell_\infty$ ($\varepsilon{=}0.3$) & $0.690\pm0.162$ & $0.799\pm0.086$ & $0.052$ & $0.91$\\
PGD-$20$ $\ell_\infty$ ($\varepsilon{=}0.3$) & $0.624\pm0.197$ & $\mathbf{0.790\pm0.086}$ & $\mathbf{0.044}$ & $0.96$\\
PGD-$20$ $\ell_2$ ($\varepsilon_2{=}0.5$) & $0.598\pm0.223$ & $0.688\pm0.177$ & $0.073$ & $0.82$\\
\bottomrule
\end{tabular}
\end{table}

\begin{figure}[t]\centering
\includegraphics[width=0.9\linewidth]{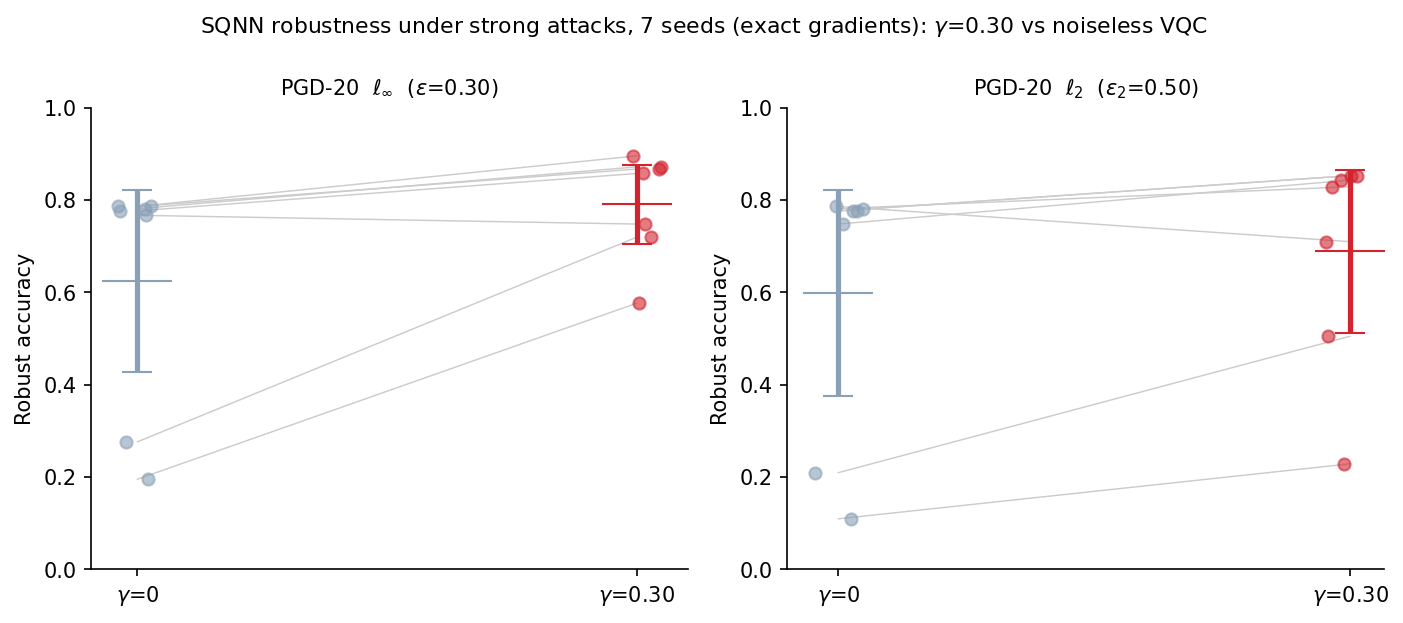}
\caption{Robustness over seven seeds under strong, exact-gradient attacks (left: $\ell_\infty$ PGD-$20$; right: $\ell_2$ PGD-$20$). Grey lines link paired seeds; markers are means with $95\%$ CI. The $\gamma=0.30$ model is more robust on $6/7$ seeds and, unlike the noiseless VQC, never collapses catastrophically.}
\label{fig:robmulti}
\end{figure}

\subsection{A second dataset and defended classical baselines (UNSW-NB15)}\label{sec:unsw}
Two questions remain: does the training-time robustness effect generalise beyond NSL-KDD, and how does the SQNN compare to \emph{defended} classical models under a \emph{white-box} attack (removing the transfer-attack handicap of Table~\ref{tab:adv})? We repeat the pipeline on UNSW-NB15~\cite{moustafa2015unsw} (identical $d=4$ PCA preprocessing, a $700$-record stratified subsample, three seeds) and compare the SQNN ($\gamma=0,0.30$) against three classical baselines, each attacked white-box with \emph{exact} input gradients: a vanilla MLP, an adversarially-trained MLP (Madry PGD~\cite{madry2018pgd}), and a Gaussian-noise-injection MLP. Table~\ref{tab:unsw} and the accuracy--robustness frontier of Fig.~\ref{fig:unswpareto} report the result.

Two findings stand out, and we state the second plainly. First, \emph{the training-time effect generalises}: on UNSW-NB15 the $\gamma=0.30$ model is again more robust than the noiseless circuit (PGD-$20$ $\ell_\infty$ $0.452$ versus $0.409$; $\ell_2$ $0.387$ versus $0.318$), confirming that the boundary-reshaping mechanism of Section~\ref{sec:strong} is not specific to one dataset. Second, \emph{the SQNN does not beat strong classical defences}: adversarial training is the most robust model ($0.706$ under $\ell_\infty$), and even a white-box-attacked vanilla MLP ($0.553$) exceeds the SQNN here. This is consistent with our stated position that no quantum advantage is claimed (Section~\ref{sec:intro}); the SQNN's contribution is the contraction theory and the training-time mechanism, not robustness supremacy over classical defences. The frontier (Fig.~\ref{fig:unswpareto}) places adversarial training at the robust extreme and the SQNN inside it, with $\gamma=0.30$ strictly improving on $\gamma=0$. Robustness is moreover \emph{dataset-dependent}: the SQNN $\gamma=0.30$ that was competitive on NSL-KDD ($0.79$ under $\ell_\infty$, Section~\ref{sec:strong}) is weaker on UNSW-NB15 ($0.45$), a caveat that any operational claim must respect.

\begin{table}[t]\centering
\caption{Second dataset (UNSW-NB15, $d=4$ PCA, three seeds): SQNN versus \emph{defended} classical baselines, all attacked \emph{white-box} with exact gradients (PGD-$20$). Mean $\pm$ std. The training-time effect generalises ($\gamma{=}0.30>\gamma{=}0$), but classical adversarial training is the most robust model; no quantum advantage is claimed.}
\label{tab:unsw}
\begin{tabular}{lccc}
\toprule
Model & clean & PGD-$20$ $\ell_\infty$ & PGD-$20$ $\ell_2$\\
\midrule
SQNN $\gamma=0$ & $0.872\pm0.027$ & $0.409\pm0.163$ & $0.318\pm0.123$\\
SQNN $\gamma=0.30$ & $0.793\pm0.048$ & $0.452\pm0.195$ & $0.387\pm0.179$\\
MLP (white-box) & $0.878\pm0.031$ & $0.553\pm0.120$ & $0.583\pm0.130$\\
MLP adv-trained & $0.820\pm0.060$ & $\mathbf{0.706\pm0.091}$ & $\mathbf{0.670\pm0.070}$\\
MLP+Gauss & $0.869\pm0.032$ & $0.573\pm0.107$ & $0.605\pm0.101$\\
\bottomrule
\end{tabular}
\end{table}

\begin{figure}[t]\centering
\includegraphics[width=0.95\linewidth]{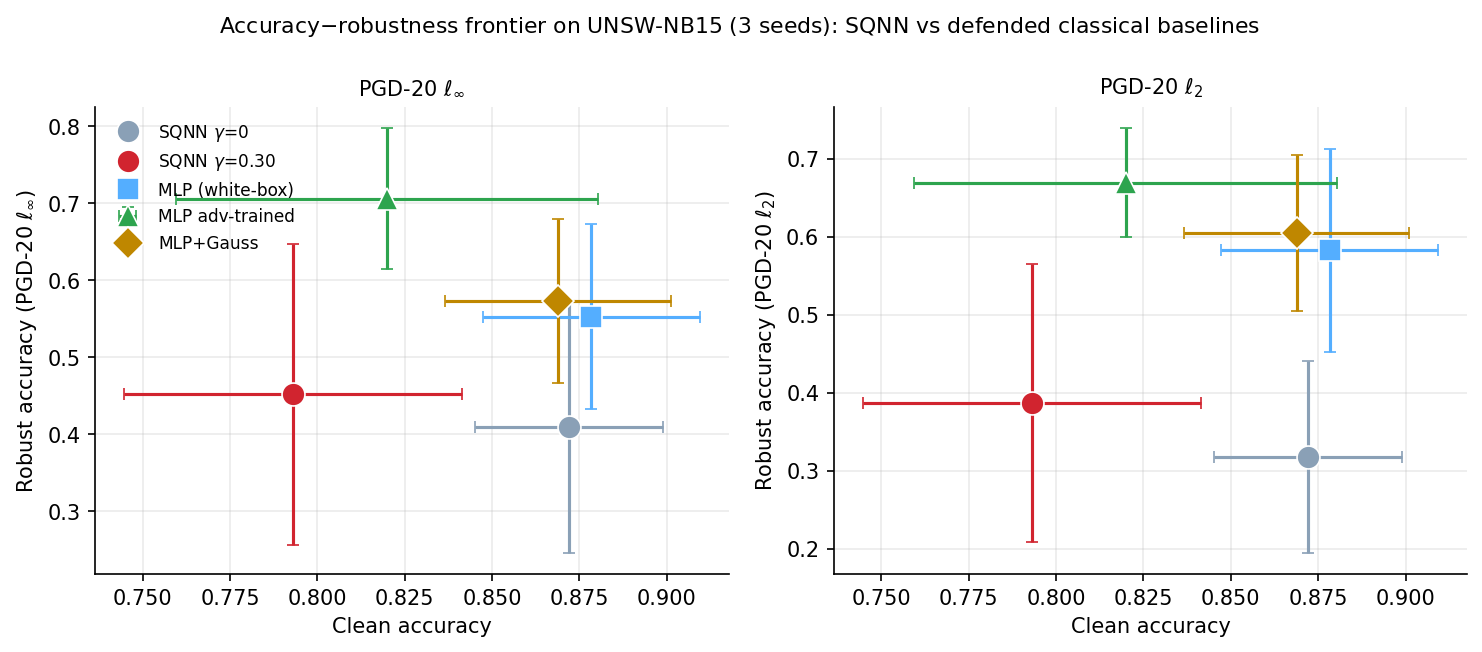}
\caption{Accuracy--robustness frontier on UNSW-NB15 (three seeds; left $\ell_\infty$, right $\ell_2$), all models attacked white-box with exact gradients. Adversarial training occupies the robust extreme; the SQNN lies inside the frontier, with $\gamma=0.30$ improving on $\gamma=0$. No quantum advantage over classical defences is claimed.}
\label{fig:unswpareto}
\end{figure}

\subsection{Noise regularisation: a small but significant effect (30 seeds)}\label{sec:reg}
Table~\ref{tab:dropout} and Fig.~\ref{fig:gap30} report the train--test gap over \emph{thirty} seeds. Both noise mechanisms reduce the gap relative to the noiseless VQC by a small but statistically significant margin: pooled across sizes ($n=90$), the depolarising channel lowers the gap by $0.009$ (paired $t$-test $p<10^{-4}$) and per-gate dropout by $0.010$ ($p=0.004$). The two are \emph{statistically indistinguishable} from each other ($p=0.74$, Fig.~\ref{fig:gap30}). The effect is concentrated where overfitting is largest --- significant at $N=40$ (Depol $p=0.001$, dropout $p=0.013$), partial at $N=80$, and vanishing at $N=160$ where there is little gap to close. The magnitude ($\approx0.01$) and its smallness are exactly the prediction of Proposition~\ref{prop:penalty}: at this scale $p(1-p)\theta^2 H$ is of order $10^{-2}$, below the seed variance ($\approx0.03$). This is why the single-seed dichotomy reported in preliminary work does not survive: with one seed the sign of the difference is essentially noise. Thirty seeds are required to resolve a real effect of this size.

\begin{table}[t]\centering
\caption{Train--test accuracy gap (mean $\pm$ std over \emph{30 seeds}; positive $=$ overfitting) on NSL-KDD. Both noise mechanisms reduce the gap; they are statistically equivalent (pooled $p=0.74$).}
\label{tab:dropout}
\begin{tabular}{lccc}
\toprule
Training size $N$ & VQC (no noise) & Depolarising & Quantum dropout\\
\midrule
40  & $+0.087\pm0.030$ & $+0.072\pm0.026$ & $+0.070\pm0.034$\\
80  & $+0.049\pm0.028$ & $+0.039\pm0.028$ & $+0.041\pm0.030$\\
160 & $+0.019\pm0.018$ & $+0.016\pm0.018$ & $+0.013\pm0.034$\\
\midrule
pooled mean ($n{=}90$) & $0.0515$ & $0.0424$ ($p{<}10^{-4}$) & $0.0414$ ($p{=}0.004$)\\
\bottomrule
\end{tabular}
\end{table}

\begin{figure}[t]\centering
\begin{minipage}{0.49\linewidth}\centering
\includegraphics[width=\linewidth]{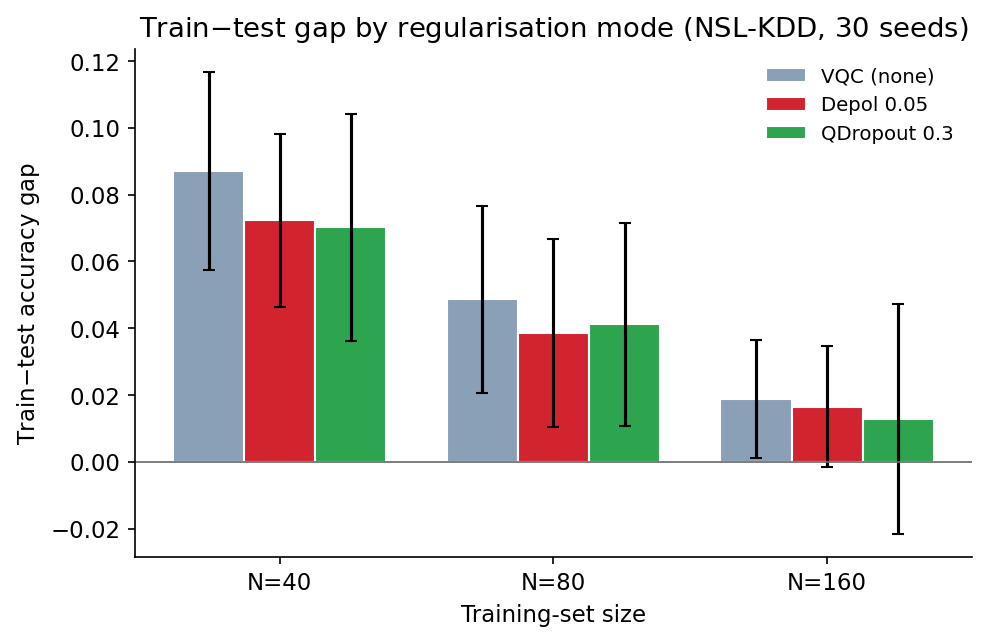}
\end{minipage}\hfill
\begin{minipage}{0.49\linewidth}\centering
\includegraphics[width=0.93\linewidth]{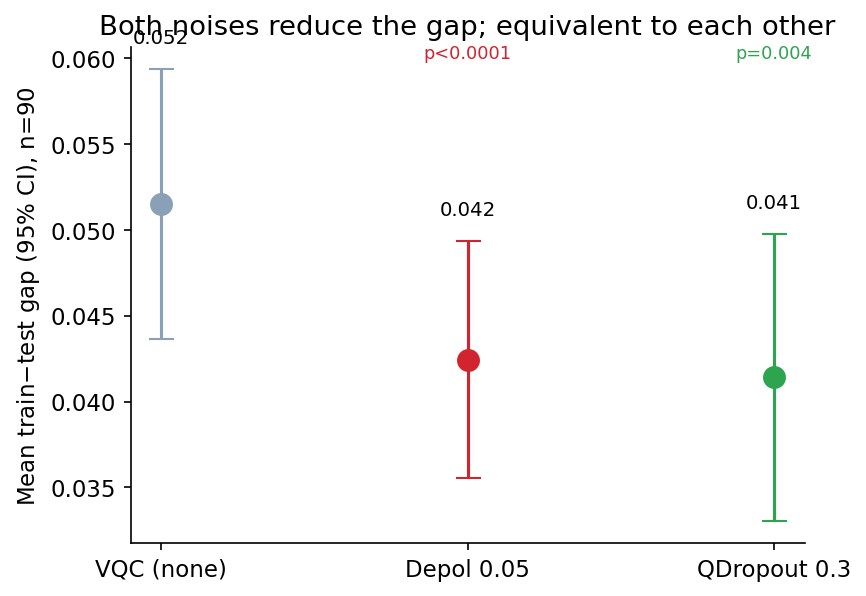}
\end{minipage}
\caption{Generalisation gap over $30$ seeds. \emph{Left:} gap by training size and mode; effects are small and shrink as $N$ grows. \emph{Right:} pooled mean gap with $95\%$ CI --- both depolarising noise and per-gate dropout reduce the gap significantly and are indistinguishable from each other.}
\label{fig:gap30}
\end{figure}

\subsection{Increasing the dropout rate does not help (depth 5)}\label{sec:sweep}
Proposition~\ref{prop:penalty} predicts that the regularisation strength peaks at $p=1/2$. We test this at depth $5$ (more variational parameters, hence more overfitting), $N=40$, $10$ seeds, tracking train \emph{and} test accuracy (Fig.~\ref{fig:sweep}). Raising the dropout rate from $0.5$ to $0.7$ does not reduce the gap (it stays at $\approx0.08$, indistinguishable from the VQC's $0.088$); instead it lowers \emph{both} train and test accuracy together (test $0.907\!\to\!0.872\!\to\!0.848$). Strong dropout thus degrades the model rather than regularising it, precisely because $p(1-p)$ decreases past $p=1/2$ while the mean contraction $(1-p)\theta$ deepens. The depolarising channel retains the lowest gap ($0.067$) while keeping test accuracy high ($0.908$), consistent with its output-space mechanism.

\begin{figure}[t]\centering
\includegraphics[width=0.66\linewidth]{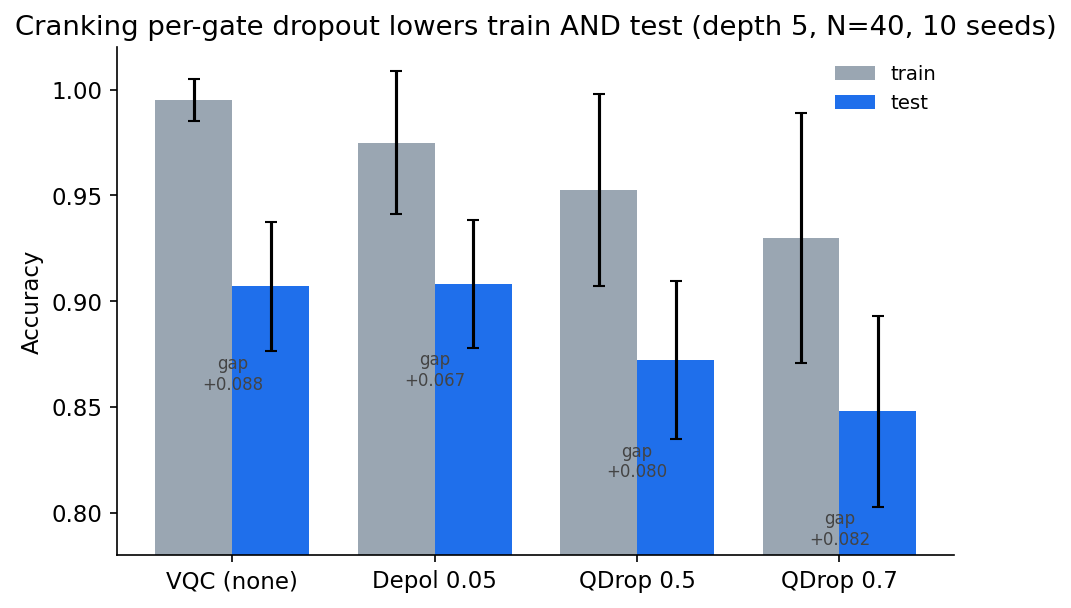}
\caption{Cranking per-gate dropout (depth $5$, $N=40$, $10$ seeds): increasing $p$ lowers train \emph{and} test accuracy together, leaving the gap unchanged --- the regularisation does not strengthen past $p=1/2$, as Proposition~\ref{prop:penalty} predicts.}
\label{fig:sweep}
\end{figure}

\subsection{Decoherence-resolved robustness}
Fig.~\ref{fig:decoherence} isolates the channel dependence predicted by Theorem~\ref{thm:contraction}: under pure dephasing --- the $\mathcal D[\sigma_z]$ dissipator aligned with the read-out --- accuracy is preserved up to large noise, whereas under isotropic depolarising noise it collapses past a critical strength. The separation is the channel-resolved content of the contraction law: $\sigma_z$-aligned decoherence sets no threshold, isotropic decoherence a depth-dependent one.

\begin{figure}[t]\centering
\includegraphics[width=0.6\linewidth]{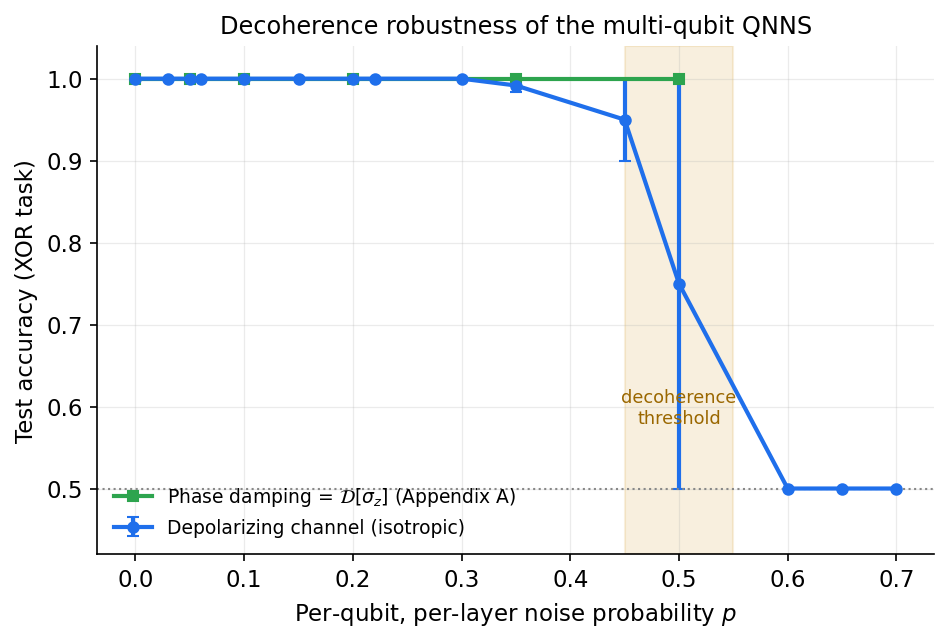}
\caption{Channel-resolved robustness: dephasing ($\mathcal D[\sigma_z]$) preserves the $\sigma_z$-basis signal, while isotropic depolarising noise contracts all Pauli components and collapses accuracy past a threshold, consistent with Theorem~\ref{thm:contraction}.}
\label{fig:decoherence}
\end{figure}

\section{Discussion}\label{sec:discussion}

The results cohere around the contraction law and the penalty formula, with an important correction to the naive reading of the former. Theorem~\ref{thm:contraction} contracts the data-dependent read-out by $\eta^{L}$, yet for sign-based $\ell_\infty$ attacks this contraction cancels under the sign (Section~\ref{sec:defence}), so it is not the source of the measured robustness. The seven-seed strong-attack study (Section~\ref{sec:strong}) shows instead that training \emph{with} the channel reshapes the decision boundary into a flatter, less attackable one: the $\gamma=0.30$ model is significantly more robust under exact-gradient PGD-$20$ and, most reliably, never suffers the catastrophic collapse that befalls the noiseless model on a sizeable fraction of seeds. This reframes the noise term of~\cite{filardo2026collab} from an empirical knob into a training-time regulariser of the boundary, and clarifies that hardening requires the stochasticity to be present at training, not merely at inference. For generalisation, Proposition~\ref{prop:penalty} shows that per-gate dropout and depolarising noise regularise in different spaces (weight versus output), but that both penalties are small at NISQ scale; the $30$-seed study confirms a real, significant, and mutually equivalent gap reduction of $\approx0.01$, and explains why prior single-seed claims were unstable. The honest reading is not that one noise type ``wins'', but that noise regularisation of SQNNs is a small, quantifiable effect whose magnitude and optimal rate the formula predicts.

For collaborative intrusion detection the implications are concrete: the ring closure carries the non-local signal at minimal circuit cost~\cite{filardo2026collab}; a depolarising channel present \emph{at training} buys stable robustness with a sub-two-point clean cost; and the variance reduction it induces makes a defended detector behave predictably across deployments. Looking ahead, the natural and highest-stakes extension is to behavioural detection in tool-using-agent pipelines, where evasion is a system-level compromise and where, by the analogy of Section~\ref{sec:background}, instruction-level manipulations such as prompt injection are an adversarial problem cousin to the one studied here. Extending SQNN anomaly detection to agent execution traces --- with its own dataset of traces, protocol, and baselines --- is left as future work; we make no agent-level claim from the present experiments.

\section{Limitations and threats to validity}\label{sec:limitations}

(i) A single simulation environment and a $d=4$ PCA feature space; multi-seed sweeps at larger $d$ and $N$ are the next step, and the $\ell_2$ robustness advantage (Table~\ref{tab:strong}), though of large effect size, is not yet significant at seven seeds. (ii) Non-differentiable classical models are attacked by transfer, understating their vulnerability; white-box attacks would widen the gap to the SQNN. Because the simulated channel is deterministic (density matrix), the model's gradients are exact and not obfuscated, so the strong PGD-$20$ attack is a valid adaptive test here; a shot-based or hardware deployment that \emph{samples} the channel at inference should additionally be stress-tested with Expectation-over-Transformation. (iii) The depolarising channel is an idealisation; on neutral-atom hardware dephasing dominates, preserving the clean signal and shifting the trade-off favourably but requiring on-device validation. (iv) Exact density-matrix simulation is feasible only to $N\approx12$ (Table~\ref{tab:feasibility}); larger networks need trajectory or tensor-network surrogates with a bond-dimension caveat. (v) We evaluate on two real datasets (NSL-KDD and UNSW-NB15); validation on the full corpora, on CIC-DDoS2019, and at larger $d$ and $N$ remains for future work. No quantum speed-up or accuracy advantage over classical baselines is claimed --- indeed, on UNSW-NB15 classical adversarial training is the most robust model (Section~\ref{sec:unsw}).

\section{Conclusion}\label{sec:conclusion}

We have given an $N$-qubit theory of stochastic quantum neural networks and used it to resolve two questions left open by the preceding collaborative-detection study. A decoherence-contraction theorem gives the law $(1-4\gamma/3)^{wL}$ for how the channel scales Pauli read-outs; we show this contraction does not, by itself, explain $\ell_\infty$ robustness (the attack's sign cancels it), and a seven-seed strong-attack study locates the effect instead in a training-time reshaping of the boundary: the depolarising SQNN is significantly more robust under exact-gradient PGD-$20$ and never suffers the catastrophic collapse that the noiseless model and classical detectors do. An adaptive-penalty formula shows that per-gate dropout implements a curvature-weighted weight-space penalty maximised at $p=1/2$, and a thirty-seed study confirms its predictions: both depolarising noise and per-gate dropout produce a small ($\approx0.01$) but significant and mutually equivalent reduction of the generalisation gap, increasing the dropout rate past one half does not help, and the single-seed dichotomy of prior work does not replicate. Together with a neutral-atom realisation and a feasibility-by-$N$ characterisation, these results turn the SQNN into a quantitatively grounded, honestly bounded architecture for adversarially robust intrusion detection.

\section*{CRediT authorship contribution statement}
\textbf{Gautier-Edouard Filardo:} Conceptualization; Methodology; Software; Formal analysis; Investigation; Writing -- original draft; Writing -- review \& editing.

\section*{Declaration of competing interest}
The author declares no known competing financial interests or personal relationships that could have appeared to influence the work reported in this paper.

\section*{Data availability}
NSL-KDD is publicly available~\cite{tavallaee2009nslkdd}. The source code reproducing all experiments---density-matrix SQNN, depolarising channel, per-gate quantum dropout, FGSM/PGD attacks, the depth-$5$ dropout-rate sweep, and the neutral-atom reservoir---is publicly available at \url{https://github.com/gautierfilardo-efrei/sqnn-adversarial} and archived at \url{https://doi.org/10.5281/zenodo.20786870}.

\end{document}